\documentclass{article}
\usepackage{spconf,amsmath,amssymb,graphicx}
\usepackage{mathtools}
\usepackage{booktabs,makecell}
\usepackage{pifont}
\usepackage{algorithm,algpseudocode}
\usepackage{amsthm}
\usepackage[hidelinks]{hyperref}

\newcommand{\yes}{\ding{51}}
\newcommand{\no}{\ding{55}}

\newcommand{\nrnodes}{n}
\newcommand{\nodeidx}{i}
\newcommand{\clusteridx}{c}
\newcommand{\sampleidx}{r}
\newcommand{\nrcluster}{k}
\newcommand{\graph}{\mathcal{G}}
\newcommand{\nodes}{\mathcal{V}}
\newcommand{\edges}{\mathcal{E}}
\newcommand{\neighbourhood}[1]{\mathcal{N}^{(#1)}}
\newcommand{\weights}{\boldsymbol{\theta}}
\newcommand{\featurevec}{\mathbf{x}}
\newcommand{\meanvec}[1]{\boldsymbol{\mu}^{(#1)}}
\newcommand{\covmtx}[1]{\boldsymbol{\Sigma}^{(#1)}}
\newcommand{\localdataset}[1]{\mathcal{D}^{(#1)}}
\newcommand{\localsamplesize}[1]{m_{#1}}
\newcommand{\locallossfunc}[2]{L_{#1}(#2)}
\newcommand{\regparam}{\alpha}
\newcommand{\simplex}[1]{\Delta^{#1}}
\newcommand{\mvnormal}[2]{\mathcal{N}(\featurevec;\,#1,\,#2)}
\newcommand{\defeq}{\coloneqq}
\newcommand{\lrate}{\eta}
\newcommand{\featuredim}{d}
\newcommand{\err}{\boldsymbol{\mu}_{\mathrm{err}}}
\newcommand{\vy}{\mathbf{y}}
\newcommand{\kld}[2]{D^{(\rm KL)}\big( #1,#2 \big)}
\newcommand{\kldapprox}[2]{\widehat{D}^{(\rm KL)}\big( #1,#2 \big)}

\newtheorem{assumption}{Assumption}
\newtheorem{proposition}{Proposition}

\title{Federated Soft Clustering via Generalized Total Variation Minimization}

\name{Shamsiiat Abdurakhmanova, Alexander Jung\thanks{Funded by the
Research Council of Finland (Decision \#363624), the Jane and Aatos Erkko
Foundation (Decision \#A835), and Business Finland. The authors thank
Ekkehard Schnoor and Salvatore Rastelli for their careful reading of
the manuscript.}}
\address{Department of Computer Science, Aalto University, Espoo, Finland}

\begin{document}
\ninept
\maketitle

\begin{abstract}
We study federated soft clustering over federated learning (FL) networks of devices 
that each hold a private local dataset and fit a personalized Gaussian mixture model (GMM).
Generalized total variation minimization (GTVMin) couples the local maximum
likelihood problems through a graph regularizer that penalizes a discrepancy
between the models of connected nodes. The choice of discrepancy measure is a
key design decision: we compare a squared Euclidean distance between model
parameters, which requires component matching, with two measures that
compare the local model distributions directly and hence need no
matching: a Monte-Carlo approximated Kullback--Leibler (KL)
divergence and a closed-form maximum mean discrepancy (MMD). 
All three resulting GTVMin instances are optimized by synchronous projected
gradient updates; for the smooth MMD instance we provide a convergence
guarantee to stationary points. We characterize their computational cost and evaluate their robustness to data heterogeneity.
\end{abstract}


\section{Introduction}
\label{sec:intro}

Clustering groups data points into groups, or clusters, of mutually
similar points~\cite{bishop2006pattern}; soft clustering assigns each point
a degree of belonging to several clusters, often through a Gaussian
mixture model (GMM), with each component representing one
cluster~\cite{bishop2006pattern,dempster1977maximum}. This paper studies
\emph{federated} soft clustering: nodes hold private local datasets, each
often too small to fit a reliable model on its own (the scarce-data regime
of Sec.~\ref{sec:experiments}), and collaborate by exchanging only model
parameters with their neighbors in an FL network.

Our approach instantiates generalized total variation minimization
(GTVMin), a design principle for federated learning (FL) developed in prior
work~\cite{ClusteredFLTVMinTSP,JungFLBook,yang2026federatedkmeansnetworks}:
the sum of the local losses is augmented with a generalized total
variation (GTV) regularizer penalizing the discrepancy between the models
of connected nodes, balancing local fit against neighbor agreement.

Our \textbf{contributions}: (i) three GTVMin instantiations for federated
soft clustering using discrepancy measures for GMMs with complementary
trade-offs (Sec.~\ref{sec:discrepancy}); (ii) a projected-gradient
algorithm with a stationarity guarantee for the MMD instance
(Sec.~\ref{sec:algorithm}, Proposition~\ref{prop:convergence}); (iii)
synthetic and handwritten-digits experiments analyzing the robustness to data heterogeneity
(Sec.~\ref{sec:experiments}).

\section{Problem Setup}
\label{sec:setup}

For a positive integer $\nrnodes$, let $[\nrnodes]\defeq\{1,\ldots,\nrnodes\}$.
Node $\nodeidx\in[\nrnodes]$ stores a local dataset
$\localdataset{\nodeidx}=\{\featurevec^{(\nodeidx,1)},\ldots,
\featurevec^{(\nodeidx,\localsamplesize{\nodeidx})}\}\subseteq\mathbb{R}^{\featuredim}$.
Nodes communicate over an undirected, unweighted graph
$\graph=(\nodes,\edges)$ with $\nodes=[\nrnodes]$, edge set
$\edges$, and neighborhoods
$\neighbourhood{\nodeidx}=\{\nodeidx':\{\nodeidx,\nodeidx'\}\in\edges\}$.
Edges act as a proxy for statistical similarity of local datasets; when
the FL network merely mirrors physical connectivity,
edges can connect dissimilar nodes --- a setting that
Sec.~\ref{sec:experiments} probes with a stochastic block model (SBM).

Each node maintains a personalized $\nrcluster$-component GMM with parameters
$\weights^{(\nodeidx)}=\{(\pi^{(\nodeidx,\clusteridx)},
\meanvec{\nodeidx,\clusteridx},\covmtx{\nodeidx,\clusteridx})\}_{\clusteridx=1}^{\nrcluster}$,
where $\pi^{(\nodeidx)}\in\simplex{\nrcluster}$, the probability
simplex in $\mathbb{R}^{\nrcluster}$, and local density
$p(\featurevec;\weights^{(\nodeidx)})=\sum_{\clusteridx=1}^{\nrcluster}
\pi^{(\nodeidx,\clusteridx)}
\mvnormal{\meanvec{\nodeidx,\clusteridx}}{\covmtx{\nodeidx,\clusteridx}}$.
We identify $\weights^{(\nodeidx)}$ with the vector in
$\mathbb{R}^{\nrcluster(1+\featuredim+\featuredim^{2})}$ that stacks
$\pi^{(\nodeidx,\clusteridx)}$, $\meanvec{\nodeidx,\clusteridx}$, and
$\mathrm{vec}\big(\covmtx{\nodeidx,\clusteridx}\big)$ for
$\clusteridx\in[\nrcluster]$; throughout, $\|\cdot\|_2$ denotes the
Euclidean norm of such stacked vectors (i.e., the Frobenius norm on
covariance blocks).
To exclude degenerate mixtures (vanishing weights, singular
covariances), we restrict the parameters to a compact convex set
$\mathcal{W}$, which also enables the convergence analysis of
Sec.~\ref{sec:algorithm}.

\begin{assumption}[Bounded parameter domain]\label{assn:bounded}
All GMM parameters lie in a compact convex set $\mathcal{W}$: there are
constants $\pi_{\min}>0$, $B_\mu<\infty$,
$0<\sigma_{\min}\le\sigma_{\max}<\infty$ with
$\pi^{(\nodeidx,\clusteridx)}\ge\pi_{\min}$,
$\|\meanvec{\nodeidx,\clusteridx}\|_2\le B_\mu$, and
$\sigma_{\min}^2 I\preccurlyeq\covmtx{\nodeidx,\clusteridx}\preccurlyeq\sigma_{\max}^2 I$
for all $\nodeidx\in[\nrnodes]$, $\clusteridx\in[\nrcluster]$.
\end{assumption}

We collect the parameters of all nodes in
$\Theta\defeq(\weights^{(1)},\ldots,\weights^{(\nrnodes)})$ and write
$\Theta\in\mathcal{W}$ when $\weights^{(\nodeidx)}\in\mathcal{W}$ for
all $\nodeidx\in[\nrnodes]$.
The local loss is the average negative log-likelihood (NLL)
\begin{equation}\label{eq:local_nll}
  \locallossfunc{\nodeidx}{\weights^{(\nodeidx)}}
  \defeq
  -\frac{1}{\localsamplesize{\nodeidx}}
  \sum_{\sampleidx=1}^{\localsamplesize{\nodeidx}}
  \ln p\big(\featurevec^{(\nodeidx,\sampleidx)};\weights^{(\nodeidx)}\big).
\end{equation}
GTVMin couples the local maximum likelihood problems via
\begin{equation}
\label{eq:gtvmin}
  \min_{\Theta\in\mathcal{W}}
  \underbrace{\sum_{\nodeidx=1}^{\nrnodes}
  \Big[
    \locallossfunc{\nodeidx}{\weights^{(\nodeidx)}}
  + \regparam \!\!\sum_{\nodeidx' \in \neighbourhood{\nodeidx}}\!\!
  D\big(\weights^{(\nodeidx)},\weights^{(\nodeidx')}\big)
  \Big]}_{\eqqcolon F(\Theta)},
\end{equation}
with regularization parameter $\regparam\ge 0$; we refer to
$F(\Theta)$ as the GTVMin objective (it depends on the choice of
$D$). For $\regparam=0$,
\eqref{eq:gtvmin} decouples into independent local GMM fits. For
$\regparam\to\infty$, connected nodes are forced to agree under the chosen
discrepancy $D$, and \eqref{eq:gtvmin} becomes a graph-coupled barycenter
problem~\cite{agueh2011barycenters}.

\paragraph*{Related work}
A single \emph{global} GMM can be fit by distributed variants of
expectation--maximization (EM) that exchange sufficient statistics sequentially~\cite{Nowak2003}, via
gossip~\cite{Kowalczyk2005}, consensus filters~\cite{Gu2008}, or
alternating optimization~\cite{forero2011distributed};
we share their decentralized setting but fit a separate GMM per node.
Split-and-conquer methods fit local GMMs but merge them at a server in
one
round~\cite{liu2014distributed,zhang2022distributed,pettersson2025federated};
our nodes refine their models over repeated neighbor exchanges.
Federated EM variants handle heterogeneity and partial participation but
rely on a central server~\cite{dieuleveut2111federated}; the gradient EM
of~\cite{tian2023towards} fits node-specific mixtures with finite-sample
guarantees, yet requires all node parameters near one common center;
our nodes may form blocks with markedly different GMMs.
In~\cite{managoli2025robust}, each client draws from one component of
a global GMM and robustly estimates only its mean; our nodes learn
full personalized GMMs.
Clustered FL partitions \emph{clients} into a few groups that each share
one model~\cite{Smith2017,marfoq2021mixture,ghosh2021}; we do not
pre-specify such groups.
\cite{wu2023fedgmm}'s unsupervised algorithm (shared components,
personalized weights, closed-form EM) coincides with the GMM
specialization of~\cite{marfoq2021mixture}.
Minimax thresholds for \emph{one-shot} transfer-assisted clustering
were recently established in~\cite{chakraborty2026transfer}.
In contrast to all of these, GTVMin~\eqref{eq:gtvmin} trains a
\emph{personalized} GMM per node, coupled by penalizing discrepancies
across the edges of the FL network.

\section{Discrepancy Measures}
\label{sec:discrepancy}

We consider three choices for $D$ in \eqref{eq:gtvmin}; their properties
are summarized in Table~\ref{tab:gtv_compare}.

\textbf{Euclidean (GTV-Eucl).}
Since GMM components are unordered, we minimize the squared parameter
distance over component permutations,
\begin{equation}\label{eq:eucl_min}
D_{\mathrm{Eucl}}\big(\weights^{(\nodeidx)},\weights^{(\nodeidx')}\big)
\defeq\min_{P\in\mathcal{S}_{\nrcluster}}
\big\|\weights^{(\nodeidx)}-P\weights^{(\nodeidx')}\big\|_2^2,
\end{equation}
where $\mathcal{S}_{\nrcluster}$ is the set of permutations of
$[\nrcluster]$ and $P\in\mathcal{S}_{\nrcluster}$ acts on
$\weights^{(\nodeidx')}$ by permuting the component triples
$(\pi^{(\nodeidx',\clusteridx)},\meanvec{\nodeidx',\clusteridx},\covmtx{\nodeidx',\clusteridx})$.
For each edge, the minimization in \eqref{eq:eucl_min} is a linear
assignment problem, solved exactly by the Hungarian
algorithm~\cite{kuhn1955hungarian}. The stacked-vector distance weights
the blocks $\pi$, $\boldsymbol{\mu}$, $\boldsymbol{\Sigma}$ equally --- adequate in our
low-dimensional regimes; block-weighted variants fit the same
framework. Fixing the
neighbors and the matching, the node-wise barycenter problem
in \eqref{eq:gtvmin} is strictly convex,
with the average of the matched neighbor parameters as unique
minimizer: GTV-Eucl induces consensus averaging. Even with an exact
matching, dissimilar neighbor components make the consensus average a
poor compromise --- the setting our SBM experiments expose.

\textbf{KL divergence (GTV-KL).}
The choice
\begin{equation*}
D_{\mathrm{KL}}\big(\weights^{(\nodeidx)},\weights^{(\nodeidx')}\big)
\defeq\kld{p(\cdot;\weights^{(\nodeidx')})}{p(\cdot;\weights^{(\nodeidx)})}
\end{equation*}
compares \emph{distributions} and is therefore invariant to component
permutations (and to differing component counts). Both orderings of
each edge appear in \eqref{eq:gtvmin}, so the per-edge penalty sums to
the \emph{symmetrized} KL divergence, while each node only ever
samples from its neighbors' models. As the KL
divergence between GMMs is intractable, we use the Monte-Carlo estimate
\begin{equation}\label{eq:kl_mc}
  \kldapprox{p^{(\nodeidx')}}{p^{(\nodeidx)}}
  \defeq
  \frac{1}{\localsamplesize{\mathrm{nbr}}}\!\sum_{\sampleidx=1}^{\localsamplesize{\mathrm{nbr}}}
  \log\frac{p^{(\nodeidx')}(\featurevec^{(\sampleidx)})}{p^{(\nodeidx)}(\featurevec^{(\sampleidx)})},
  \;\;
  \featurevec^{(\sampleidx)}\!\sim\! p^{(\nodeidx')},
\end{equation}
so each node effectively fits its GMM to its local data augmented with
synthetic samples drawn from its neighbors' current models.

\textbf{MMD (GTV-MMD).}
The maximum mean discrepancy compares distributions in a
kernel-induced feature space~\cite{JMLR:v13:gretton12a,dziugaite2015}. With the kernel inner product
$\langle p,q\rangle_k\defeq\mathbb{E}[k(\featurevec,\vy)]$,
$\featurevec\sim p$, $\vy\sim q$,
\begin{equation*}
  \mathrm{MMD}^2_\sigma(p,q)
  =\langle p,p\rangle_k-2\langle p,q\rangle_k+\langle q,q\rangle_k .
\end{equation*}
For Gaussian factors
$\featurevec\sim\mathcal{N}(\boldsymbol{\mu},\boldsymbol{\Sigma})$,
$\vy\sim\mathcal{N}(\boldsymbol{\mu}',\boldsymbol{\Sigma}')$ and the
radial basis function kernel
$k(\featurevec,\vy)=\exp(-\|\featurevec-\vy\|_2^2/2\sigma^2)$ the kernel
expectation has the closed form
\begin{equation}\label{eq:mmd}
  \mathbb{E}[k(\featurevec,\vy)]
  = \sigma^{\featuredim}\,|S|^{-1/2}
  \exp\big(-\tfrac{1}{2}\Delta\mu^{\!\top} S^{-1}\Delta\mu\big),
\end{equation}
with $\Delta\mu=\boldsymbol{\mu}-\boldsymbol{\mu}'$ and
$S=\boldsymbol{\Sigma}+\boldsymbol{\Sigma}'+\sigma^2 I$: the two
component covariances enter only through their sum, widened by the
kernel bandwidth. Eq.~\eqref{eq:mmd} extends to mixtures by summing
\eqref{eq:mmd} over component pairs weighted by the mixture weights.
$D_{\mathrm{MMD}}\defeq\mathrm{MMD}^2_\sigma(p^{(\nodeidx)},p^{(\nodeidx')})$
is thus available in closed form: evaluating it and its gradient requires
no sampling, no component matching, and no equal component counts. Since a fixed bandwidth $\sigma$
yields vanishing gradients between distant components, we use a
multi-scale kernel: the sum of \eqref{eq:mmd} over the bandwidth grid
specified in Sec.~\ref{sec:experiments}.

\begin{table}[t]
	\centering
	\caption{Properties of the three GTV discrepancy measures to compare
	local GMMs. Per-edge cost: leading-order operation count of one
	evaluation of $D$; $\nrcluster$ components, dimension $\featuredim$,
	Monte-Carlo sample size $\localsamplesize{\mathrm{nbr}}$.}
	\label{tab:gtv_compare}
	\footnotesize
	\setlength{\tabcolsep}{3pt}
	\begin{tabular}{@{}lcccc@{}}
		\toprule
		GTV & Per-edge cost & \makecell{Sample-\\free} & \makecell{Matching-\\free} & \makecell{Variable\\components} \\
		\midrule
		Eucl & $\nrcluster^2 \featuredim^2 + \nrcluster^3$ & \yes & \no  & \no  \\
		KL   & $\nrcluster \featuredim^3 + \localsamplesize{\mathrm{nbr}} \nrcluster \featuredim^2$ & \no  & \yes & \yes \\
		MMD  & $\nrcluster^2 \featuredim^3$ & \yes & \yes & \yes \\
		\bottomrule
	\end{tabular}
\end{table}

\section{Algorithm and Convergence}
\label{sec:algorithm}

We solve all three GTVMin instances with the same algorithmic engine:
synchronous projected gradient descent (Algorithm~\ref{alg:gtvmin}). In each communication round, every node
performs $T_{\mathrm{loc}}$ local gradient steps on its instance of
\eqref{eq:gtvmin} with the neighbors' parameters frozen at their last
communicated values (Jacobi-style parallel updates), then broadcasts its
parameters to its neighbors. Each local step is
\begin{equation}\label{eq:pgd_step}
\weights^{(\nodeidx)}\!\leftarrow\!
\Pi_{\mathcal{W}}\Big(\weights^{(\nodeidx)}\!-\lrate\Big[
\nabla\locallossfunc{\nodeidx}{\weights^{(\nodeidx)}}
+\regparam\!\!\sum_{\nodeidx'\in\neighbourhood{\nodeidx}}\!\!
\mathbf{g}^{(\nodeidx,\nodeidx')}\Big]\Big),
\end{equation}
where $\Pi_{\mathcal{W}}$ projects blockwise onto $\mathcal{W}$, which
factorizes over nodes and components: project $\pi^{(\nodeidx)}$ onto
the simplex with floor $\pi_{\min}$, rescale
$\meanvec{\nodeidx,\clusteridx}$ to norm at most $B_\mu$, and clip the
eigenvalues of $\covmtx{\nodeidx,\clusteridx}$ to
$[\sigma_{\min}^2,\sigma_{\max}^2]$ --- each a closed-form operation;
$\mathbf{g}^{(\nodeidx,\nodeidx')}$ is the gradient of
the frozen-neighbor discrepancy: for GTV-Eucl,
$2(\weights^{(\nodeidx)}-P_{\nodeidx,\nodeidx'}\weights^{(\nodeidx',t-1)})$
with the matched permutation $P_{\nodeidx,\nodeidx'}$; for GTV-KL, the
NLL gradient of node $\nodeidx$'s GMM on the
$\localsamplesize{\mathrm{nbr}}$ neighbor samples --- the same form as
$\nabla\locallossfunc{\nodeidx}{\cdot}$, evaluated on synthetic data;
for GTV-MMD, the gradient of
$\langle p,p\rangle_k-2\langle p,p'\rangle_k$, obtained in closed form
by differentiating \eqref{eq:mmd}.

\begin{algorithm}[t]
	\caption{GTVMin federated soft clustering}
	\label{alg:gtvmin}
	\begin{algorithmic}[1]
		\Require graph $\graph=(\nodes,\edges)$; local datasets $\{\localdataset{\nodeidx}\}$;
		components $\nrcluster$; discrepancy $D$ from
		Sec.~\ref{sec:discrepancy} (Eucl~\eqref{eq:eucl_min}; KL~\eqref{eq:kl_mc};
		MMD~\eqref{eq:mmd});
		parameters $\regparam,\lrate$; rounds $T$; local steps $T_{\mathrm{loc}}$
		\State each node $\nodeidx$: $\weights^{(\nodeidx,0)} \gets$ local
		$\nrcluster$-means fit on $\localdataset{\nodeidx}$;
		$\locallossfunc{\nodeidx}{\cdot} \gets$ NLL \eqref{eq:local_nll}
		of $\localdataset{\nodeidx}$
		\For{$t=1$ to $T$}
		\State per-neighbor state from frozen $\weights^{(\nodeidx',t-1)}$:
		Eucl re-matches components (Hungarian on \eqref{eq:eucl_min}); KL draws
		$\localsamplesize{\mathrm{nbr}}$ samples per neighbor for
		\eqref{eq:kl_mc}
		\For{each node $\nodeidx\in\nodes$ in parallel}
		\State $\weights^{(\nodeidx,t)} \gets$ $T_{\mathrm{loc}}$ projected
		gradient steps, from $\weights^{(\nodeidx,t-1)}$, on
		$\locallossfunc{\nodeidx}{\weights^{(\nodeidx)}}
		+\regparam\sum_{\nodeidx'\in\neighbourhood{\nodeidx}}
		D(\weights^{(\nodeidx)},\weights^{(\nodeidx',t-1)})$
		\EndFor
		\State each node sends $\weights^{(\nodeidx,t)}$ to $\neighbourhood{\nodeidx}$
		\EndFor
		\Ensure personalized GMM parameters $\{\weights^{(\nodeidx,T)}\}_{\nodeidx\in\nodes}$
	\end{algorithmic}
\end{algorithm}

We now give a convergence guarantee for GTV-MMD; the closing paragraph
of this section discusses GTV-KL and GTV-Eucl. For GTV-MMD, the
objective $F$ of \eqref{eq:gtvmin} uses $D=D_{\mathrm{MMD}}$:
by the analyticity of \eqref{eq:mmd}, $F$ is a single smooth function
of $\Theta$ --- no sampling and no matching enter its evaluation.

On $\mathcal{W}$ (Assumption~\ref{assn:bounded}),
each NLL term \eqref{eq:local_nll} is twice continuously differentiable
with a bounded Hessian, hence has a Lipschitz gradient
(cf.~\cite[Lem.~1.2.2]{nesterov2014lectures}); the coupling term is
analytic with bounded second derivatives on the compact $\mathcal{W}$.
Thus $F$ is $\beta$-smooth on $\mathcal{W}$ for some $\beta<\infty$.
One round of Algorithm~\ref{alg:gtvmin}
($T_{\mathrm{loc}}=1$, full gradients) is the projected gradient step
$\Theta^{(t+1)}=\Pi_{\mathcal{W}}(\Theta^{(t)}-\lrate\nabla
F(\Theta^{(t)}))$, with gradient mapping
$\mathcal{G}^{(t)}\defeq\tfrac{1}{\lrate}(\Theta^{(t)}-\Theta^{(t+1)})$;
$\mathcal{G}^{(t)}=\mathbf{0}$ exactly when $\Theta^{(t)}$ is a
stationary point of \eqref{eq:gtvmin}.

\begin{proposition}\label{prop:convergence}
Under Assumption~\ref{assn:bounded}, with $T_{\mathrm{loc}}=1$ and step
size $\lrate<1/\beta$, the iterates of Algorithm~\ref{alg:gtvmin} for
GTV-MMD satisfy
\begin{equation*}
  \min_{0\le t<T}\|\mathcal{G}^{(t)}\|_2^2
  \le \frac{1}{T}\sum_{t=0}^{T-1}\|\mathcal{G}^{(t)}\|_2^2
  \le \frac{2\big(F(\Theta^{(0)})-F^{*}\big)}{\lrate\,T},
\end{equation*}
where $F^{*}=\min_{\Theta\in\mathcal{W}}F(\Theta)$.
\end{proposition}

\begin{proof}[Proof sketch]
The descent lemma~\cite[Lem.~1.2.3]{nesterov2014lectures} for the
$\beta$-smooth $F$, combined
with the variational inequality of the
projection~\cite[Prop.~2.1.3]{bertsekas1999nonlinear},
gives, whenever $\lrate<1/\beta$,
\begin{equation*}
F(\Theta^{(t+1)})\le F(\Theta^{(t)})
-\tfrac{\lrate}{2}\|\mathcal{G}^{(t)}\|_2^2 .
\end{equation*}
Telescoping and $F\ge F^{*}$ (attained on the compact $\mathcal{W}$)
yield the bound.
\end{proof}

For GTV-KL, the Monte-Carlo
surrogate \eqref{eq:kl_mc} with frozen neighbors is unbiased and smooth
on $\mathcal{W}$, yielding descent of each round's surrogate; a
cross-round guarantee, with fresh samples, is left open. For GTV-Eucl, the per-round re-matching renders the objective
piecewise smooth; a majorize--minimize extension is left to an
extended version. For $T_{\mathrm{loc}}>1$, later local steps use
stale neighbor parameters; the $T_{\mathrm{loc}}=5$ used in the
experiments is a communication-saving heuristic, also open.

\section{Experiments}
\label{sec:experiments}

\textbf{Setup.}
Unless stated otherwise, we use $\nrnodes=10$ nodes on
Erd\H{o}s--R\'enyi (ER) graphs with edge
probability $p=0.6$ (plus a ring backbone so no node is isolated),
$T=40$ rounds,
$T_{\mathrm{loc}}=5$ full-batch gradient steps, learning rate
$\lrate=0.02$, $\localsamplesize{\mathrm{nbr}}=50$, MMD bandwidths
$\sigma\in\{0.5,1,2,4\}$, dimension
$\featuredim=2$, and $\nrcluster=2$ components for synthetic data. Every experiment is repeated over $20$
seeds; we report medians and $95\%$ percentile-bootstrap confidence
intervals (CIs) over seeds. Baselines: Local
(independent GMM per node) and Centralized (one GMM on the pooled data),
both via scikit-learn~\cite{pedregosa2011sklearn}
(\texttt{covariance\_type=full}, \texttt{n\_init=5}), and FedKMeans
(synchronous federated $k$-means with Hungarian-aligned center
averaging), a hard-clustering baseline isolating whether GTVMin's
distributional coupling improves on simple federated centroid
averaging on the likelihood-free metrics ($\err$, NMI). Metrics: validation log-likelihood (LL), normalized mutual
information (NMI), and centroid estimation error
$\err=\tfrac{1}{\nrnodes\nrcluster}\sum_{\nodeidx,\clusteridx}
\|\meanvec{\clusteridx}-\meanvec{\nodeidx,\clusteridx}\|_2$ after Hungarian
alignment, where $\meanvec{\clusteridx}$ denotes the true cluster
centroids. $\regparam$ is chosen per algorithm by grid search over
$\{10^{-1},\ldots,10^{2}\}$ on homogeneous-data validation LL ($10^0$,
$10^{-1}$, $10^0$ for GTV-Eucl/KL/MMD); in practice, $\regparam$ can
instead be chosen by cross-validation on local data. Code reproducing
all results accompanies the paper.

\begin{table}[t]
	\centering
	\caption{Synthetic datasets ($\localsamplesize{\nodeidx}=100$,
	$\featuredim=2$, $\nrcluster=2$): median [95\% bootstrap CI] over 20
	seeds. Best per column and dataset (incl.\ ties) in bold.}
	\label{tab:syn_data}
	\footnotesize
	\setlength{\tabcolsep}{4pt}
	\begin{tabular}{lcc}
		\toprule
		Algorithm & LL & $\err$ \\
		\midrule
		\multicolumn{3}{l}{\emph{Dataset} $\mathcal{D}_{\text{iso}}$}\\
		GTV-Eucl (ours) & $\mathbf{-3.52}$ $[-3.53,-3.50]$ & $0.06$ $[0.04,0.07]$ \\
		GTV-KL (ours) & $-3.54$ $[-3.55,-3.52]$ & $0.12$ $[0.11,0.14]$ \\
		GTV-MMD (ours) & $-3.54$ $[-3.55,-3.51]$ & $0.08$ $[0.07,0.10]$ \\
		Local & $-3.57$ $[-3.59,-3.54]$ & $0.18$ $[0.17,0.20]$ \\
		Centralized & $\mathbf{-3.52}$ $[-3.53,-3.50]$ & $\mathbf{0.05}$ $[0.04,0.07]$ \\
		FedKMeans & --- & $0.06$ $[0.04,0.08]$ \\
		\midrule
		\multicolumn{3}{l}{\emph{Dataset} $\mathcal{D}_{\text{var}}$}\\
		GTV-Eucl (ours) & $\mathbf{-3.52}$ $[-3.54,-3.49]$ & $0.08$ $[0.05,0.11]$ \\
		GTV-KL (ours) & $-3.54$ $[-3.56,-3.51]$ & $0.14$ $[0.13,0.16]$ \\
		GTV-MMD (ours) & $-3.53$ $[-3.56,-3.51]$ & $0.10$ $[0.08,0.12]$ \\
		Local & $-3.58$ $[-3.59,-3.53]$ & $0.19$ $[0.18,0.20]$ \\
		Centralized & $\mathbf{-3.52}$ $[-3.54,-3.49]$ & $\mathbf{0.06}$ $[0.05,0.08]$ \\
		FedKMeans & --- & $0.08$ $[0.05,0.13]$ \\
		\midrule
		\multicolumn{3}{l}{\emph{Dataset} $\mathcal{D}_{\text{aniso}}$}\\
		GTV-Eucl (ours) & $\mathbf{-3.26}$ $[-3.28,-3.24]$ & $0.06$ $[0.04,0.08]$ \\
		GTV-KL (ours) & $-3.28$ $[-3.30,-3.26]$ & $0.13$ $[0.12,0.16]$ \\
		GTV-MMD (ours) & $-3.29$ $[-3.30,-3.26]$ & $0.08$ $[0.07,0.10]$ \\
		Local & $-3.32$ $[-3.33,-3.29]$ & $0.19$ $[0.17,0.20]$ \\
		Centralized & $\mathbf{-3.26}$ $[-3.28,-3.24]$ & $\mathbf{0.05}$ $[0.04,0.06]$ \\
		FedKMeans & --- & $0.06$ $[0.04,0.08]$ \\
		\bottomrule
	\end{tabular}
\end{table}

\textbf{Homogeneous synthetic data.}
Local datasets are drawn from a shared $\nrcluster$-component GMM (means
$\sim\mathcal{N}(\mathbf{0},25 I)$) with isotropic
($\mathcal{D}_{\mathrm{iso}}$), varying-scale
($\mathcal{D}_{\mathrm{var}}$), and anisotropic
($\mathcal{D}_{\mathrm{aniso}}$) covariance structures. Each node holds
$\localsamplesize{\nodeidx}=100$ training and $400$ validation
data points. Table~\ref{tab:syn_data}: every GTV variant improves
over Local on validation LL and $\err$ on all three datasets.
GTV-Eucl tracks Centralized within $0.001$ nats in LL and $0.014$ in
$\err$: consensus averaging matches identical local distributions.
The distributional couplings lose $\approx0.02$ nats more. FedKMeans
is competitive on $\err$ but has no likelihood model. Median NMI is $1.00$ for
every method, so Table~\ref{tab:syn_data} omits it.
\emph{Takeaway: on homogeneous data, consensus-style GTV-Eucl is
essentially as good as pooling all data.}

\textbf{Heterogeneity: clustered FL networks.}
We model heterogeneity with a three-block SBM: $\nrnodes=15$ nodes split into three blocks of five; edges
appear with probability $p_{\mathrm{in}}=0.8$ within a block and
$p_{\mathrm{out}}$ across any pair of blocks, plus a chain within
each block for connectivity. The cluster means of
blocks 2 and 3 are each shifted by an independently drawn
$\nu_b\sim\mathcal{N}(\mathbf{0},25I)$, so every block is internally
homogeneous ($\mathcal{D}_{\mathrm{iso}}$). We set $\nrcluster=5$ and
reduce the local sample size to $\localsamplesize{\nodeidx}=50$ (10
points per component): unlike the ample-data first experiment, this
probes the scarce-data regime where pooling matters (cf.\
Table~\ref{tab:ext_base}), re-selecting $\regparam$ by the grid
protocol. Fig.~\ref{fig:sbm}
shows the sweep over $p_{\mathrm{out}}$. When the FL network matches
the cluster structure ($p_{\mathrm{out}}=0$), every GTVMin variant
realizes block-wise pooling: GTV-KL attains a median LL of $-4.38$
(train $-4.18$), ahead of overfitting Local training ($-5.15$; train
$-3.9$) and Centralized ($-5.44$, which merges the blocks). As $p_{\mathrm{out}}$ grows the
couplings separate. GTV-Eucl collapses abruptly: its unbounded
quadratic penalty pulls equally hard regardless of neighbor distance,
so $\err$ jumps from $0.46$ at $p_{\mathrm{out}}=0$ to $3.98$ already
at $p_{\mathrm{out}}=0.1$.
GTV-KL degrades gradually, falling behind Local between
$p_{\mathrm{out}}=0.1$ ($\err=1.14$) and $0.2$ ($\err=1.51$). GTV-MMD
degrades most gracefully and stays best across the sweep, even at
$p_{\mathrm{out}}=0.8$ where the graph carries no cluster information ($\mathrm{LL}=-4.95$, $\err=0.85$): the bounded RBF kernel in \eqref{eq:mmd} yields
vanishing gradients on edges connecting distant models, suppressing
boundary-crossing edges automatically. NMI: GTV-Eucl falls to
$0.73$; GTV-KL/-MMD hold $0.84$--$0.88$.
Fig.~\ref{fig:sbm_rounds} exposes the mechanism at
$p_{\mathrm{out}}=0.2$ over rounds: cross-block edges pull GTV-KL and
GTV-Eucl up, while GTV-MMD holds $\err\approx0.6$.
\emph{Takeaway: GTVMin pools statistical strength within blocks; the
bounded MMD kernel makes GTV-MMD robust to spurious cross-block
edges; the unbounded couplings are not.}

\begin{figure}[t]
	\centering
	\includegraphics[width=0.96\linewidth]{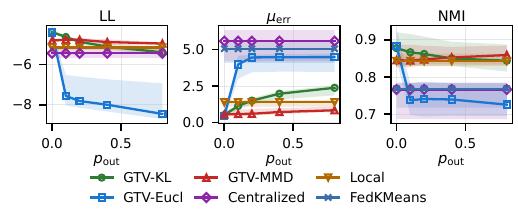}
	\caption{Three-block SBM with block-homogeneous data: LL (left),
	$\err$ (middle), and NMI (right) vs.\ the inter-block edge
	probability $p_{\mathrm{out}}$; $\mathcal{D}_{\mathrm{iso}}$,
	$\nrcluster=5$, $\localsamplesize{\nodeidx}=50$. Median and 95\%
	bootstrap CI over 20 seeds.}
	\label{fig:sbm}
\end{figure}

\begin{figure}[t]
	\centering
	\includegraphics[width=0.92\linewidth]{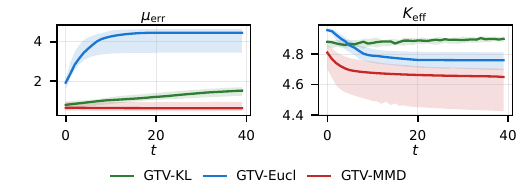}
	\caption{Centroid error (left) and effective component count
	$K_{\mathrm{eff}}^{(\nodeidx)}
	=\exp(-\sum_{\clusteridx}\pi^{(\nodeidx,\clusteridx)}
	\log\pi^{(\nodeidx,\clusteridx)})$ (right) over communication rounds
	at $p_{\mathrm{out}}=0.2$; $K_{\mathrm{eff}}$ stays close to
	$\nrcluster$. Median and 95\% CI over 20 seeds.}
	\label{fig:sbm_rounds}
\end{figure}

\textbf{Comparison with prior federated GMM methods.}
Table~\ref{tab:ext_base} compares GTV-KL/-MMD with five published
methods (our
implementations)~\cite{Gu2008,marfoq2021mixture,pettersson2025federated,tian2023towards,ghosh2021}
on the same data, protocol, and round budget as
Algorithm~\ref{alg:gtvmin}; FedGrEM's hyperparameters are jointly
tuned by the same oracle grid as $\regparam$, and we extend its
isotropic-only reference implementation to full covariance, matching
our setting; IFCA is graph-agnostic and receives the true SBM 
block count as an oracle input, unlike the other methods, including
GTV-KL/-MMD; GTV-Eucl is omitted
(Fig.~\ref{fig:sbm} shows its collapse). On a homogeneous ER network,
FedGrEM/FedEM/FedGenGMM and (trivially, given its oracle) IFCA reduce
to (near-)global models matching Centralized, ahead of Local; the
graph-aware consensus EM~\cite{Gu2008} instead trails Centralized by
$0.45$ nats here, since our $T=40$-round budget is too short for its
single-exchange-per-round filter to converge even within one
connected block. On the SBM ($p_{\mathrm{out}}=0.2$), every
method without oracle cluster knowledge merges the blocks
($\err\geq4.3$) --- including the personalized
FedGrEM~\cite{tian2023towards}, whose common-center shrinkage fails
here, and consensus EM~\cite{Gu2008}, whose filter diffuses toward a
single network-wide average regardless of block structure --- while
GTV-MMD attains the best LL and $\err$ among these and GTV-KL the
second-best LL. IFCA, handed the true block count, separates the
blocks almost as well (LL $-4.43$, $\err=0.98$) as GTV-MMD, but only
as an oracle bound.
\emph{Takeaway: without knowing the true number of clusters,
global-model and common-center baselines merge the blocks;
graph-coupled personalization does not.}

\begin{table}[t]
	\centering
	\caption{Prior federated GMM methods vs.\ GTV-KL/-MMD: homogeneous
	ER network and three-block SBM ($p_{\mathrm{out}}=0.2$),
	$\nrcluster=5$, $\localsamplesize{\nodeidx}=50$; median [95\% CI]
	over 20 seeds. Best per column (incl.\ ties) in bold. IFCA is given
	the true number of blocks as an oracle input.}
	\label{tab:ext_base}
	\scriptsize
	\renewcommand{\arraystretch}{0.94}
	\setlength{\tabcolsep}{0.9pt}
	\begin{tabular}{lccc}
		\toprule
		Algorithm & LL (homog.) & LL (SBM) & $\err$ (SBM) \\
		\midrule
		GTV-KL (ours) & $-4.36$ {\tiny$[-4.42,-4.24]$} & $-4.87$ {\tiny$[-4.98,-4.79]$} & $1.50$ {\tiny$[1.24,1.68]$} \\
		GTV-MMD (ours) & $-4.57$ {\tiny$[-4.67,-4.47]$} & $-4.77$ {\tiny$[-4.82,-4.69]$} & $\mathbf{0.60}$ {\tiny$[0.48,0.93]$} \\
		Local & $-5.09$ {\tiny$[-5.17,-4.94]$} & $-5.15$ {\tiny$[-5.28,-5.07]$} & $1.39$ {\tiny$[1.10,1.56]$} \\
		FedGrEM~\cite{tian2023towards} & $-4.36$ {\tiny$[-4.44,-4.25]$} & $-5.45$ {\tiny$[-5.60,-5.26]$} & $5.15$ {\tiny$[4.28,5.49]$} \\
		Cons.\ EM~\cite{Gu2008} & $-4.77$ {\tiny$[-4.87,-4.59]$} & $-5.62$ {\tiny$[-5.79,-5.36]$} & $4.34$ {\tiny$[3.75,5.05]$} \\
		FedEM~\cite{marfoq2021mixture} & $-4.36$ {\tiny$[-4.43,-4.25]$} & $-5.13$ {\tiny$[-5.34,-4.87]$} & $5.28$ {\tiny$[4.39,5.61]$} \\
		FedGenGMM~\cite{pettersson2025federated} & $-4.34$ {\tiny$[-4.42,-4.23]$} & $-5.45$ {\tiny$[-5.73,-5.19]$} & $5.49$ {\tiny$[4.18,5.98]$} \\
		IFCA~\cite{ghosh2021} & $\mathbf{-4.32}$ {\tiny$[-4.39,-4.21]$} & $\mathbf{-4.43}$ {\tiny$[-4.46,-4.30]$} & $0.98$ {\tiny$[0.34,1.39]$} \\
		\bottomrule
	\end{tabular}
\end{table}

\textbf{Handwritten digits.}
We embed the $1797$ images ($8\times8$ pixels) of the scikit-learn
handwritten digits dataset~\cite{pedregosa2011sklearn} into $\featuredim=10$
dimensions by PCA, rescaled to unit overall standard deviation, and
set $\nrcluster=10$. We distribute the data over $\nrnodes=10$ nodes by a
Dirichlet label split ($\alpha=0.5$: strong skew; $\alpha=100$:
near-i.i.d.), holding out $40\%$ per node for validation (ER graph,
$p=0.6$, $T=25$, $\lrate=0.01$; $\regparam$ from the synthetic grid
protocol).
Fig.~\ref{fig:digits} shows LL and NMI over $\alpha$: with only
$\approx110$ training points per node and $\nrcluster=10$
full-covariance components, Local overfits badly (median LL $-32$ at
$\mathrm{Dir}(\alpha)=1$, train LL $-3.5$); every GTV variant lifts
the LL by $\approx20$ nats (GTV-KL: $-12.4$, train $-8.7$);
Centralized retains the best LL/NMI ($-9.9$, $0.83$--$0.84$).
GTV-KL attains the best federated LL under skew ($-12.1$,
$\alpha=0.5$); GTV-Eucl nears Centralized toward i.i.d.\ ($-12.0$,
$\alpha=100$). On NMI, parameter-space methods (GTV-Eucl $0.79$,
FedKMeans $0.80$) beat distributional couplings ($0.61$--$0.66$).
\emph{Takeaway: under label skew, distributional couplings win on
likelihood, parameter-space couplings on NMI.}

\begin{figure}[t]
	\centering
	\includegraphics[width=0.92\linewidth]{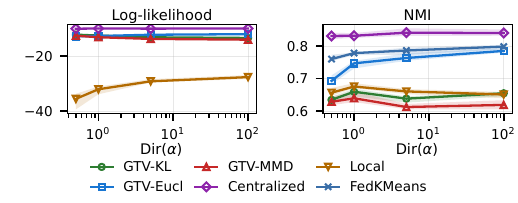}
	\caption{Handwritten digits: validation LL and NMI vs.\ Dirichlet
	concentration $\mathrm{Dir}(\alpha)$ of the label split. Median and
	95\% bootstrap CI over 20 seeds; FedKMeans has no likelihood model (NMI
	only).}
	\label{fig:digits}
\end{figure}

\section{Conclusion}
\label{sec:conclusion}

We cast federated soft clustering as GTVMin over an FL network of
personalized GMMs; the GTV discrepancy is the central design decision,
with a stationarity guarantee for MMD
(Proposition~\ref{prop:convergence}); the cheap Euclidean coupling
(Table~\ref{tab:gtv_compare}) needs matching and, under
heterogeneity, drags models toward a poor average. The matching-free distributional couplings realize
block-wise pooling on clustered FL networks, beating all baselines;
the bounded MMD kernel loses the least log-likelihood across cluster
boundaries.

\section{Compliance with Ethical Standards}
This work uses only synthetic and public benchmark data (no human
subjects). Computational resources: Aalto Science-IT. Text and code
were developed with the assistance of a large language model
(Anthropic Claude); the authors verified all methods and results.

\bibliographystyle{IEEEbib}
\bibliography{local_refs}

\end{document}